\documentclass[twoside]{article}

\usepackage[accepted]{aistats2019}
\usepackage{float}
\usepackage{subcaption}
\usepackage{amsmath}
\usepackage{algorithmic}
\usepackage{algorithm}
\usepackage{amsthm}
\usepackage{amsfonts}
\usepackage{comment}
\usepackage{graphicx}
\usepackage{graphics}
\usepackage{color}
\usepackage{epstopdf}

\newtheorem{theorem} {Theorem}
\newtheorem{lemma} {Lemma}
\newtheorem{definition} {Definition}

\newtheorem{assumption} {Assumption}

\def\a{{\mathbf{a}}}
\def\c{{\mathbf{c}}}

\def\x{{\mathbf{x}}}
\def\g{{\mathbf{g}}}
\def\u{{\mathbf{u}}}
\def\v{{\mathbf{v}}}

\def\w{{\mathbf{w}}}
\def\y{{\mathbf{y}}}

\def\q{{\mathbf{q}}}
\def\b{{\mathbf{b}}}

\def\C{{\mathbf{C}}}
\def\A{{\mathbf{A}}}
\def\M{{\mathbf{M}}}
\def\N{{\mathbf{N}}}
\def\I{{\mathbf{I}}}

\def\C{{\mathbf{C}}}

\def\Q{{\mathbf{Q}}}

\DeclareMathOperator*{\argmin}{arg\,min}

\newcommand{\mP}{\mathcal{P}}
\newcommand{\mK}{\mathcal{K}}

\newcommand{\mX}{\mathcal{X}}

\newcommand{\mD}{\mathcal{D}}
\newcommand{\E}{\mathbb{E}}

\newcommand{\regret}{\textrm{regret}}
\newcommand{\dist}{\textrm{dist}}

\newcommand{\reals}{\mathbb{R}}

\begin{document}

%

%

\twocolumn[

\aistatstitle{Logarithmic Regret for Online Gradient Descent Beyond Strong Convexity}

\aistatsauthor{Dan Garber}

\aistatsaddress{Technion - Israel Institute of Technology}

]

\begin{abstract}
Hoffman's classical result gives a bound on the distance of a point from a convex and compact polytope in terms of the magnitude of violation of the constraints. Recently, several results showed that Hoffman's bound can be used to derive strongly-convex-like rates for first-order methods for \textit{offline} convex optimization of curved, though not strongly convex, functions, over polyhedral sets.
In this work, we use this classical result for the first time to obtain faster rates for \textit{online convex optimization} over polyhedral sets with curved convex, though not strongly convex, loss functions. We show that under several reasonable assumptions on the data, the standard \textit{Online Gradient Descent} algorithm guarantees logarithmic regret. To the best of our knowledge, the only previous algorithm to achieve logarithmic regret in the considered settings is the \textit{Online Newton Step} algorithm which requires quadratic (in the dimension) memory and at least quadratic runtime per iteration, which greatly limits its applicability to large-scale problems. In particular, our results hold for  \textit{semi-adversarial} settings in which the data is a combination of  an arbitrary  (adversarial) sequence and a stochastic sequence, which might provide reasonable approximation for many real-world sequences, or under a natural assumption that the data is low-rank. We demonstrate via experiments that the regret of OGD is indeed comparable to that of ONS (and even far better)  on curved though not strongly-convex losses.
\end{abstract}

\section{Introduction}
The celebrated \textit{Online Gradient Descent} algorithm (OGD), originally due to \cite{Zinkevich03}, is a natural adaptation of the classical projected (sub)gradient descent algorithm for \textit{offline} convex optimization, to the setting of \textit{Online Convex Optimization} \cite{Hazan16,SSS12}. The benefits of OGD are two folded: (i) in many problems of interest it performs very efficient iterations, which can often be executed in linear time (in the dimension), and (ii) it often guarantees optimal regret rates, mainly in terms of the length of the sequence $T$, e.g., $\sqrt{T}$ regret for arbitrary convex loss functions, and $\log{T}$ regret in case all loss functions are strongly-convex \cite{AHK07}. 

However, there exists a highly-important and wide class of loss functions, known as \textit{exp-concave losses} \cite{AHK07}, for which OGD does not guarantee optimal regret (in terms of $T$). For instance, the family of exp-concave losses capture important problems such as online linear regression with the square loss and online LASSO, online logistic regression, online portfolio selection, and more.
While for exp-concave losses, OGD only guarantees regret that scales like $\sqrt{T}$, it is known that an online algorithm known as \textit{Online Newton Step} (ONS), originally due to \cite{AHK07}, guarantees $\log{T}$ regret (see also recent work \cite{Luo6} which gives an improved variant in terms of runtime and regret bound for low-rank data). On the downside, while OGD applies very efficient iterations in terms of runtime and memory requirements (e.g., when computing the gradient vector of the loss function and projecting onto the feasible set is computationally-cheap), ONS requires quadratic memory and to solve a linear system on each iteration (which requires at least quadratic runtime via efficient implementation). ONS also requires to compute a non-Euclidean projection on each iteration to enforce the constraints, which can be considerably more expensive than the Euclidean projection required by OGD (e.g., might require to use an iterative algorithm). Thus, despite the improved regret bound, ONS is often not applicable to large-scale problems.  This naturally motivates the following question:
\begin{center}
\textit{Can Online Gradient Descent be shown to enjoy a logarithmic regret bound for classes of loss functions beyond the class of strongly convex losses?}
\end{center}


In this paper we take a step forward towards understanding the conditions under-which OGD can guarantee logarithmic regret, hence yielding both an efficient algorithm and an improved convergence rate in such settings.
In particular, we focus on an important sub-class of the exp-concave losses: loss functions which can be written as a strongly-convex function applied to a linear transformation of the input variables. Such loss functions include important examples such as the square loss for linear regression, the online portfolio optimization loss, the logistic regression loss, and more. While such losses are not necessarily strongly convex in the entire space, they are strongly convex on a certain subspace, which corresponds to the row-span of the linear transformation.
Our main result shows that when all loss functions are of this form, with linear transformations that satisfy certain consistency conditions, and the feasible set is a convex and compact polytope, the vanilla OGD algorithm, with a suitable choice of learning rate, indeed guarantees logarithmic regret.
To the best of our knowledge, this is the first result to establish strongly-convex-like rates for OGD without strong convexity, on an important and wide class of applications.


Technically, at the heart of our result lies a classical result in convex analysis, originally due to Hoffman \cite{Hoffman52}, which roughly speaking, bounds the distance of a point from a convex and compact polytope in terms of the magnitude of violation of the constraints describing the polytope. In case the feasible set is a polytope and the loss function is as described above (i.e., strongly convex applied to linear transformation), it can be shown that Hoffman's bound implies a property known as \textit{quadratic growth}, which upper bounds the $\ell_2$ distance between any feasible point and a feasible optimal solution, in terms of the distance in function values - a property well known to enable faster convergence rates in convex optimization settings (this is often also the main consequence of strong convexity needed in order to achieve fast rates for strongly-convex optimization).

Indeed, several recent works have used this classical result by Hoffman \cite{Hoffman52}, to achieve fast rates without strong convexity for \textit{offline optimization} problems,  see for instance the recent works \cite{Necoara16, wang2014, Karimi16, Beck15,Xu17}. Importantly, all of these results consider only \textit{stationary} settings, in which the objective function is fixed. As we show in the sequel, obtaining such fast rates results in the \textit{online convex optimization} setting is considerably more challenging since, as opposed to strong convexity which is a property that holds in the entire space, and hence, given a sequence of strongly convex functions, this property holds throughout the sequence, Hoffman's bound on the other-hand, is related to a specific subspace (which corresponds to the row-span of the linear transformation in the losses discussed above), and thus, given a sequence of such losses with different corresponding subspaces, these subspaces need not be, informally speaking, consistent with each other. Hence, a main contribution of this work is to formalize and analyze conditions under which this property could indeed be leveraged towards obtaining fast rates in a non-stationary online setting.

In particular, we show that our logarithmic-regret result holds for sequences which can be expressed as a combination of an arbitrary (adversarial) sequence and a stochastic sequence with certain stationary characteristics, which may potentially serve as reasonable approximation to many real-world data-streams, or when the data enjoys a low-rank structure. We report preliminary experimental results on both synthetic and real-world datasets which indeed show that OGD can outperform the Online Netwon Step method, both in terms of the regret and computational efficiency, on \textit{non-strongly convex} sequences.
 
\section{Preliminaries}

Throughout this work we use $\Vert\cdot\Vert$ to denote the Euclidean norm for vectors and the spectral norm (i.e., largest singular value) for matrices. Also, for a compact set $\mP\subset\reals^d$ and a matrix $\C\in\reals^{m\times d}$, we use the notation $\C\mP := \{\C\x ~|~\x\in\mP\}$.

\subsection{Convex optimization preliminaries}
\begin{definition}
Given a convex and compact set $\mK\subset\reals^d$ and a real-valued function $f$, differentiable over $\mK$, we say $f$ is $G$-Lipschitz over $\mK$ if $\forall \x\in\mK:$ $\Vert{\nabla{}f(\x)}\Vert \leq G$.
\end{definition}
In particular, if $f$ is convex, differentiable and $G$-Lipschitz over a convex and compact $\mK$, we have that $\forall \x,\y\in\mK$:
\begin{eqnarray}\label{eq:Lip}
f(\x)-f(\y) \leq (\x-\y)^{\top}\nabla{}f(\x) \leq G\Vert{\x-\y}\Vert.
\end{eqnarray}

\begin{definition}
Given a convex and compact set $\mK\subset\reals^d$ and a real-valued function $f$ which is differentiable over $\mK$, we say $f$ is $\alpha$-strongly convex over $\mK$ if $\forall \x,\y\in\mK$: $f(\x) \leq f(\y) + (\x-\y)^{\top}\nabla{}f(\x) - \frac{\alpha}{2}\Vert{\x-\y}\Vert^2$.

\end{definition}

We recall the first-order optimality condition for convex differentiable functions (see for instance \cite{boyd2004convex}): for any $\mK\subset\reals^d$, convex and compact, and a real-valued function $f$, convex and differentiable over $\mK$, we have that $\forall \x\in\mK,~\x^*\in\argmin_{\y\in\mK}f(\y)$: $(\x^*-\x)^{\top}\nabla{}f(\x^*) \leq 0$.


\subsection{Online convex optimization preliminaries}
We now briefly recall the setting of Online Convex Optimization (OCO). For a more in-depth introduction we refer the reader to \cite{Hazan16,SSS12}.

In the OCO problem, a decision maker (DM) is required to iteratively choose points in a fixed convex and compact set $\mK \subset\reals^d$. On each round $t$, after the DM makes his choice, i.e., chooses some $\x_t\in\mK$, a convex function $f_t:\mK\rightarrow\reals$ is revealed, and the DM suffers the loss $f_t(\x_t)$. This process continues for $T$ rounds, where $T$ is assumed to be known in advanced.
The goal is to design an algorithm for choosing the actions of the DM so to minimize a quantity called \textit{regret}, which is given by
\vspace{-10pt}
\begin{eqnarray*}
\regret_T := \sum_{t=1}^Tf_t(\x_t) - \min_{\x\in\mK}\sum_{t=1}^Tf_t(\x).
\end{eqnarray*}
It is well known than an algorithm known as Online Gradient Descent, see Algorithm \ref{alg:ogd} below, can guarantee a $O(GD\sqrt{T})$ bound on the regret, where $D$ is the $\ell_2$ diameter of $\mK$ and $G$ is an $\ell_2$ upper bound on the gradients of the functions $f_1,\dots,f_T$ \cite{Zinkevich03}, which is in general optimal. It is also known that when all functions $f_1,\dots,f_T$ are $\alpha$-strongly convex, the same algorithm (though with different learning rate) guarantees $O((G^2/\alpha)\log{T})$ regret \cite{AHK07}, which is also optimal under this assumption \cite{Hazan14}.

\begin{algorithm}
\caption{Online (projected) Gradient Descent}
\label{alg:ogd}
\begin{algorithmic}[1]
\STATE $\x_1 \gets $ some arbitrary point in $\mK$
\FOR{$t=1\dots T$}
\STATE $\y_{t+1} \gets \x_t - \eta_t\nabla{}f_t(\x_t)$
\STATE $\x_{t+1} \gets \arg\min_{\x\in\mK}\Vert{\x-\y_{t+1}}\Vert^2$
\ENDFOR
\end{algorithmic}
\end{algorithm}

\subsection{Hoffman's bound and the quadratic growth property}

\begin{definition}
We say a matrix $\C\in\reals^{m\times d}$ is $\sigma$-Hoffman with respect to a convex and compact polytope $\mP\subset\reals^d$ for some $\sigma >0$, if for any vector $\c\in\reals^m$ such that the set $\mP(\C,\c) := \{\x\in\mP ~ | ~ \C\x = \c\}$ is not empty, it holds that
$\forall\x\in\mP$: $\dist(\x,\mP(\C,\c))^2 \leq \sigma^{-1}\Vert{\C\x-\c}\Vert^2$.
\end{definition}

The following Lemma, originally due to Hoffman \cite{Hoffman52}, shows that a Hoffman parameter bounded away from zero, always exists. Here we give the result in rephrased form. A proof is given in the appendix for completeness. 

\begin{lemma}\label{lem:hoffman}
Let $\mP:=\{\x\in\reals^d ~ | ~ \A\x \leq \b\}$  be a compact and convex polytope and let $\C\in\reals^{m\times d}$. Given a vector $\c\in\reals^m$, define the set $\mP(\C,\c) := \{\x\in\mP ~ | ~ \C\x=\c\}$. If $\mP(\C,\c) \neq \emptyset$, then there exists $\sigma > 0$ such that
$\forall \x\in\mP$: $\dist(\x,\mP(\C,\c))^2\leq \sigma^{-1}\Vert{\C\x - \c}\Vert^2$.
Moreover, we have the bound $\sigma \geq \min_{\Q\in\mathcal{M}}\lambda_{\min}\left({\Q\Q^{\top}}\right)$, where $\mathcal{M}$ is the set of all $d\times d$ matrices whose rows are linearly independent rows of the matrix $\M := (\A^{\top},~ \C^{\top})^{\top}$, and $\lambda_{\min}(\cdot)$ denotes the smallest non-zero eigenvalue.
\end{lemma}

\begin{definition}. 
We say a function $f:\reals^d\rightarrow\reals$ has the quadratic growth property with parameter $\kappa$ with respect to a compact and convex set $\mK\subset\reals^d$, if it holds that
$\forall\x\in\mK$: $\dist(\x,\mX^*)^2 \leq \frac{2}{\kappa}\left({f(\x) - f^*}\right)$,
where $\mX^* := \arg\min_{\y\in\mK}f(\y)$ and $f^* := \min_{\y\in\mK}f(\y)$.
\end{definition} 

The following lemma, which will be instrumental in the proof of our main result, demonstrates the connection between Hoffman's bound and the quadratic growth property for (possibly stochastic) convex objectives. 
\begin{lemma}[from Hoffman's bound to quadratic growth]\label{lem:hoffmanExpectation}
Let $\mP\subset\reals^d$ be a convex and compact polytope.
Let $\mD$ be a distribution over pairs $(g(\cdot),\C)\in(\reals^m\rightarrow\reals)\times\reals^{m\times d}$ satisfying:
\begin{enumerate}
\item for each pair $(g,\C)$ in the support of $\mD$, the function $g$ is differentiable, $G$-Lipschitz, and $\alpha$-strongly convex over $\C\mP$.
\item the function $F(\x):=\E_{(g,\C)\sim\mD}\left[{g(\C\x)}\right]$ is differentiable over $\mP$.
\item  the expectation $\E_{(g(\cdot),\C)\sim\mD}[\C^{\top}\C]$ exists.
\end{enumerate}
Let $\C_{\mD}\in\reals^{k\times d}$ be such that  $\C_{\mD}^{\top}\C_{\mD} = \E_{(g(\cdot),\C)\sim\mD}[\C^{\top}\C]$, and denote by $\sigma$ the Hoffman constant of $\C_{\mD}$ w.r.t. $\mP$. Finally, define $\mX^*:=\arg\min_{\y\in\mP}\{F(\y):=\E_{(g,\C)\sim\mD}\left[{g(\C\y)}\right]\}$. Then, there exists $\c_{\mD}\in\reals^k$ such that $\x\in\mX^*$ $\Longleftrightarrow$ $\C_{\mD}\x = \c_{\mD}$. Moreover, $\forall \x\in\mP$: $\dist(\x,\mX^*)^2 \leq \frac{1}{\sigma}\Vert{\C_{\mD}\x - \c_{\mD}}\Vert^2 \leq \frac{2}{\alpha\sigma}\left({F(\x) - \min_{\y\in\mP}F(\y)}\right)$.
\end{lemma}

\begin{proof}
For any $\x\in\mP$ and $\x^*\in\mX^*$ it holds that 
{\small
\begin{align*}
F(\x^*) - F(\x) &= \E_{(g,\C)\sim\mD}[g(\C\x^*) - g(\C\x)] \\
&\underset{(a)}{\leq} \E_{(g,\C)\sim\mD}\Big[(\x^*-\x)^{\top}\C^{\top}\nabla{}g(\C\x^*) \\
&- \frac{\alpha}{2}\Vert{\C(\x^*-\x)}\Vert^2\Big] \\
&= (\x^*-\x)^{\top}\E_{(g,\C)\sim\mD}\left[{\C^{\top}\nabla{}g(\C\x^*)}\right] \\
&- \frac{\alpha}{2}(\x^*-\x)^{\top}\E_{(g,\C)\sim\mD}\left[{\C^{\top}\C}\right](\x^*-\x) \\
&\underset{(b)}{=} (\x^*-\x)^{\top}\nabla{}F(\x^*)  -\frac{\alpha}{2}\Vert{\C_{\mD}(\x^*-\x)}\Vert^2 \\
&\underset{(c)}{\leq}  -\frac{\alpha}{2}\Vert{\C_{\mD}(\x^*-\x)}\Vert^2,
\end{align*}}
where (a) follows since each $g(\cdot)$ in the support of $\mD$ is differentiable and $\alpha$-strongly convex over $\C\mP$, (b) follows since $F(\x)$ is differentaible over $\mP$ and hence it's gradient vector is given by $\nabla{}F(\x) = \E_{(g,\C)\sim\mD}[\frac{d}{d\x}g(\C\x)] = \E_{(g,\C)\sim\mD}[\C^{\top}\nabla{}g(\C\x)]$, and (c) follows form the first-order optimality condition for $F(\cdot)$.
Thus, $\forall \x^*\in\mX^*, \x\in\mP$: 
\begin{eqnarray}\label{eq:hoffExp:1}
\Vert{\C_{\mD}(\x^*-\x)}\Vert^2 \leq \frac{2}{\alpha}\left({F(\x) - F(\x^*)}\right).
\end{eqnarray}
Thus, setting $\c_{\mD} = \C_{\mD}\x^*$ for some $\x^*\in\mX^*$, directly gives the $\Longrightarrow$ direction of the first part of the lemma

To prove the $\Longleftarrow$ direction of the first part of the lemma, let $\x\in\mP$ such that $\C_{\mD}\x=\c_{\mD}$. Then we have that
\begin{eqnarray*}
0 &=& (\x-\x^*)\C_{\mD}^{\top}\C_{\mD}(\x-\x^*) \\
&=& \E_{(g,\C)\sim\mD}[\Vert{\C\x-\C\x^*}\Vert^2] \\
&\geq &\left({\E_{(g,\C)\sim\mD}[\Vert{\C\x-\C\x^*}\Vert]}\right)^2.
\end{eqnarray*}

Since for each pair $(g,\C)$ in the support of $\mD$, $g$ is convex and $G$-Lipschitz over $\C\mP$, using Eq. \eqref{eq:Lip} we have that
\begin{eqnarray*}
F(\x) - F(\x^*) &=& \E_{(g,\C)\sim\mD}[g(\C\x) - g(\C\x^*)] \\
&\leq &\E_{(g,\C)\sim\mD}[G\Vert{\C\x - \C\x^*}\Vert] = 0,
\end{eqnarray*}
meaning $\x\in\mX^*$, which completes the proof of the first part of the lemma. The second part of the lemma follows directly form combining the first part of the lemma with Hoffman's bound (Lemma \ref{lem:hoffman}) and Eq. \eqref{eq:hoffExp:1}.
\end{proof}

\section{Informal Statement of Results and Examples}

We now give an informal statement of our theoretical results, followed by several concrete examples to demonstrate possible applications. We then conclude the section by drawing a connection between our setting and \textit{online exp-concave optimization} and the Online Newton Step algorithm.

\subsection{Logarithmic regret for Online Gradient Descent without strong convexity}
Suppose that the feasible set $\mP$ is a convex and compact polytope in $\reals^d$ and suppose all loss functions are of the form $f_t(\x) := g_t(\C_t\x)$, where $g_t(\cdot)$ is differentiable and $\alpha_1$-strongly convex. Suppose further, that there exists a matrix $\M\in\reals^{k\times d}$ such that for all $t\in[T]$, $\M^{\top}\M\succeq \E[\C_t^{\top}\C_t] \succeq \alpha_2\M^{\top}\M$, where the expectation is with respect to possible randomness in the choice of $\C_t$. Then, we show there exists a choice of step-sizes $\{\eta_t\}_{t\in[T]}$ such that OGD guarantees $O(\log{T})$ regret (treating all other quantities as constants). 

We note that while the requirement $\M^{\top}\M\succeq\E[\C_t^{\top}\C_t] \succeq \alpha_2\M^{\top}\M$ seems not standard at first glance, observe that when $\C_t$ is full-rank, and hence $f_t(\x)$ is in particular strongly convex, this requirement holds trivially with $\M =\I$. Hence, this condition is natural for dealing with loss functions that are strongly convex only on a restricted subspace of $\reals^d$, requiring them all to be consistent, at least in expectation, with the same subspace. In Subsection \ref{sec:examples} we discuss several settings of interest in which it is reasonable to assume this requirement holds.


\subsection{Examples of relevant settings and loss functions}\label{sec:examples}

\textbf{Linear regression and Lasso:}
Consider the $\ell_p$ linear regression loss function  $f_t(\x) := \frac{1}{2}\Vert{\A_t\x - \b_t}\Vert_p^2$, with $p\in(1,2]$. In particular, when $p=2$ and the feasible polytope is an $\ell_1$-ball, i.e., $\mP := \{\x\in\reals^d ~ | ~ \Vert{\x}\Vert_1 \leq k\}$, for some $k >0$, we get an online version of the famous LASSO problem \cite{Lasso96}.

For our log-regret result to hold for \textit{deterministic data}, i.e., deterministic choice of $(\A_1,\b_1)\dots(\A_t,\b_T)$, it must hold that $\textrm{row-span}(\A_1) = \textrm{row-span}(\A_2) = \dots = \textrm{row-span}(\A_T)$. This is reasonable if $\A_t\in\reals^{m\times d}$ for a large enough value of $m$, and the data, i.e., the rows of $\A_1,\dots,\A_T$ lie in a certain low-dimensional subspace \footnote{this may be natural to assume for instance, if the data is the output of some dimension reduction technique such as the wildly used \textit{principal component analysis} procedure}.

A different \textit{non-deterministic} setting of interest is a ``semi-adversarial" model in which $\A_t := \tilde{\A}_t + \N_t$, where the matrices $\tilde{\A}_1,\dots\tilde{\A}_T$ are arbitrary and $\N_t\sim\mD$ i.i.d. for all $t\in[T]$, for some fixed (yet unknown) distribution $\mD$. Then, a sufficient condition for our fast OGD rate  to hold (in expectation), is that $\E_{\mD}[\N] = \mathbf{0}$ and $\E_{\mD}[\N^{\top}\N] \succeq \alpha\tilde{\A}_t^{\top}\tilde{\A}_t$ for all $t\in[T]$, for some $\alpha > 0$. That is, the data can be faithfully modeled as a deterministic sequence perturbed by a well-conditioned stochastic noise. For instance, such a model underlies the problem of \textit{Universal Linear Filtering} studied in \cite{Moon2009, GarberH13}.

Finally, if we can treat the data as generated by a stochastic mechanism that on each time $t$ randomly samples $\A_t$ from a (unknown) distribution $\mD_t$ (note we allow the distribution to change each round), then a sufficient condition for our log-regret result to hold is that there exists a  matrix $\M$ and $\alpha > 0$ such that for all $t\in[T]$: $\M^{\top}\M \succeq \E_{\A_t\sim\mD_t}[\A_t^{\top}\A_t] \succeq \alpha\M^{\top}\M$.

\textbf{Logistic regression:}
In online logistic regression, the loss on $m$ data points organized in a matrix $\A_t\in\reals^{m\times d}$, can be written as $f_t(\x) := \sum_{i=1}^m\log\left({1+\exp(\A_t^{(i)\top}\x)}\right)$,
where $\A_t^{(i)}$ denotes the $i$th row of the matrix $\A_t$. Observe $f_t(\x)$ can be rewritten as $f_t(\x) := g(\A_t\x)$ with 
$g(\y) := \sum_{i=1}^m\log\left({1+e^{\y_i}}\right)$. It not difficult to verify that for bounded $\y$, $g(\y)$ is indeed strongly convex, and hence this problem also falls into our setting. As in the linear regression case, if the feasible set is a polytope (e.g., standard selections are a $\ell_1$ or $\ell_{\infty}$ ball), then same assumptions on the matrices $\A_1,\dots,\A_T$ will allow to apply our log-regret result. 

\textbf{Online portfolio selection:}
In the online portfolio selection problem \cite{Hazan16}, the loss of a \textit{rebalancing} portfolio $\x$ (a point in the unit simplex) on $m$ consecutive trading rounds is given by $f_t(\x) = -\sum_{i=1}^m\log\left({\A_t^{(i)\top}\x}\right)$, where the rows of $\A_t\in\reals^{m\times d}$, $\A_t > \mathbf{0}$  (entry-wise) encodes the asset prices on each round. Similarly to the logistic regression example, we can write $f_t(\x) := g(\A_t\x)$ with $g(\y) := -\sum_{i=1}^m\log\left({\y_i}\right)$. Again, it is not hard to verify that if $\A_t \geq r$ (entry-wise) for some $r>0$, then $g(\cdot)$ is indeed strongly convex over the transformed  simplex $\A_t\Delta_d := \{\A_t\x ~ |~ \x\in\reals^d, ~\x\geq\mathbf{0},~ \sum_{i=1}^d\x_i = 1\}$. Again, our log-regret result holds under the same assumptions on the data $\A_1,\dots,\A_T$, as above.

\subsection{Connection with exp-concavity and the Online Newton Step algorithm}

A real-valued function $f$, twice-differentiable over a compact set $\mK\subset\reals^d$, is $\sigma$ exp-concave on $\mK$ if and only if
$\forall \x\in\mK:  \nabla^2{}f(\x) \succeq \sigma\nabla{}f(\x)\nabla{}f(\x)^{\top}$ \cite{Hazan16}.

Note that in case $f(\x) := g(\C\x)$, where $g$ is $\alpha$-strongly convex and twice-differentiable over $\C\mK$, denoting $G = \sup_{\x\in\C\mK}\Vert{\nabla{}g(\x)}\Vert$, we have that $\forall \x\in\mK$: {\small
\begin{align*}
\nabla^2{}f(\x) &= \C^{\top}\nabla^2g(\C\x)\C \underset{(a)}{\succeq} \alpha\C^{\top}\C \succeq \frac{\alpha}{G^2}\Vert{\nabla{}g(\C\x)}\Vert^2\C^{\top}\C \\
&\underset{(b)}{\succeq} \frac{\alpha}{G^2}\C^{\top}\nabla{}g(\C\x)\nabla{}g(\C\x)^{\top}\C = \frac{\alpha}{G^2}\nabla{}f(\x)\nabla{}f(\x)^{\top},
\end{align*}}
where (a) follows since $g(\cdot)$ is $\alpha$-strongly convex, and (b) follows since for any vector $\y$ we have that $\y^{\top}\C^{\top}\nabla{}g(\C\x)\nabla{}g(\C\x)^{\top}\C\y = ((\C\y)^{\top}\nabla{}g(\C\x))^2 \leq \Vert{\nabla{}g(\C\x)}\Vert^2\cdot\Vert{\C\y}\Vert^2 =  \Vert{\nabla{}g(\C\x)}\Vert^2\cdot\y^{\top}\C^{\top}\C\y$. Hence, $f(\x)$ is $\alpha/G^2$-exp-concave over $\mK$.  

Thus, if all loss functions are as above, i.e., $f_t(\x) := \g_t(\C_t\x)$, and we let $G$ be a uniform upper bound on the $\ell_2$ norm of the gradients of $g_t(\cdot)$ and $C$ be a uniform upper bound on the spectral norm of the matrices $\C_t$ , the Online Newton Step (ONS) algorithm \cite{Hazan16}, guarantees regret bound: $\regret_T(ONS)  = O\left({\frac{G^2}{\alpha} + CGD}\right)d\log{T}$, where $D$ is the $\ell_2$ diameter of $\mK$.


\section{Logarithmic Regret for OGD Without Strong Convexity}

In this section we present and prove our main result - a logarithmic regret bound for Online Gradient Descent (Algorithm \ref{alg:ogd}) without strong-convexity. As discussed, our result holds under certain conditions on the data which are captured in the following assumption. In the following Subsection \ref{subsec:app} we discuss several concrete examples in which this assumption holds.

\begin{assumption}\label{ass:ogd:dist}
Given a convex and compact polytope $\mP\subset\reals^d$,
a distribution $\mD$ over pairs $(g,\C)\in(\reals^m\rightarrow\reals)\times\reals^{m\times d}$, is said to satisfy Assumption \ref{ass:ogd:dist} with parameters $(\M, \bar{G}, \alpha_1, \alpha_2)\in\reals^{k\times d}\times\reals_+^3$ w.r.t. $\mP$, if it holds that
\begin{enumerate}
\item each function $g$, part of a pair $(g,\C)$ in the support of $\mD$, is differentiable, $\alpha_1$-strongly convex over $\C\mP := \{\C\x ~ | ~\x\in\mP\}$ and $G$-Lipschitz over $\C\mP$, for some finite $G>0$
\item the function $F(\x):=\E_{(g,\C)\sim\mD}\left[{g(\C\x)}\right]\}$ is differentiable over $\mP$
\item $\forall \x\in\mP$: $\E_{(g,\C)\sim\mD}\left[{\Vert{\frac{d}{d\x}g(\C\x)}\Vert^2}\right]= \E_{(g,\C)\sim\mD}\left[{\Vert{\C^{\top}\nabla{}g(\C\x)}\Vert^2}\right] \leq \bar{G}^2$
\item  the expectation $\E_{(g(\cdot),\C)\sim\mD}[\C^{\top}\C]$ exists and satisfies $\E_{(g,\C)\sim\mD}[\C^{\top}\C] \succeq \alpha_2\M^{\top}\M$.
\end{enumerate}
\end{assumption}

We can now state our main theorem, Theorem \ref{thm:ogd:master}. While the theorem holds under quite general conditions, we refer the reader again to Section \ref{sec:examples} for discussion of concrete applications.

\begin{theorem}\label{thm:ogd:master}[OGD Master Theorem]
Fix a convex and compact polytope $\mP\subset \reals^d$.
Consider a sequence of $T$ distributions $\mD_1,\dots\mD_T$ over $(\reals^m\rightarrow\reals)\times\reals^{m\times d}$ which satisfy Assumption \ref{ass:ogd:dist} with parameters $(\bar{\C},\bar{G},\alpha_1,\alpha_2)$ w.r.t. $\mP$, where $\bar{\C}$ is a matrix satisfying $\bar{\C}^{\top}\bar{\C} = \frac{1}{T}\sum_{t=1}^T\E_{(g,\C)\sim\mD_t}[\C^{\top}\C]$. Suppose further that $\bar{\C}$ is $\sigma$-Hoffman w.r.t. $\mP$.
Let $f_1(\x)\dots f_T(\x)$ be a sequence of loss functions such that $f_t(\x) = g_t(\C_t\x)$, with $(g_t,\C_t)\sim\mD_t$ independently of the functions $\{f_{\tau}\}_{\tau\in[T]\setminus\{t\}}$. Then, applying Algorithm \ref{alg:ogd} with step-size $\eta_t = \frac{1}{\alpha_1\alpha_2\sigma{}t}$, w.r.t. the losses $f_1\dots f_T$ and the polytope $\mP$, guarantees that
\vspace{-6pt}
\begin{eqnarray*}
&&\max_{\x\in\mP}\E_{f_1\sim\mD_1\dots f_T\sim\mD_T}\left[{\sum_{t=1}^Tf_t(\x_t) - \sum_{t=1}^Tf_t(\x)}\right] \leq \\
&&\frac{\alpha_1\alpha_2\sigma{}D^2}{2} + \frac{\bar{G}^2}{2\alpha_1\alpha_2\sigma}(1 + \ln{T}).
\end{eqnarray*}
\end{theorem}

Before we can prove the theorem, we need the following technical lemma which extends Lemma \ref{lem:hoffmanExpectation} from a single stochastic objective to a sequence of stochastic objectives, and hence plays a key role in our regret analysis . 

\begin{lemma}[from Hoffman's bound to quadratic growth of a sequence]\label{lem:uniformMinimizers}
Let $\mD_1,\dots,\mD_T$ be distributions over $(\reals^m\rightarrow\reals)\times\reals^{m\times d}$ satisfying Assumption \ref{ass:ogd:dist} with parameters $(\bar{\C}, \bar{G}, \alpha_1, \alpha_2)\in\reals^{k\times d}\times\reals_+^3$, where $\bar{\C}$ is a matrix satisfying $\bar{\C}^{\top}\bar{\C} = \frac{1}{T}\sum_{t=1}^T\E_{(g,\C)\sim\mD_t}[\C_t^{\top}\C_t]$.
Consider the function 
$F(\x) := \frac{1}{T}\sum_{t=1}^T\E_{(g,\C)\sim\mD_t}[g(\C\x)]$,
and define the set of feasible minimizers:
$\mX^* := \argmin_{\x\in\mP}F(\x)$.
Then, it holds that $\forall \x^*,\y^*\in\mX^*, ~ t\in[T]$: $\E_{(g,\C)\sim\mD_t}[g(\C\x^*)] = \E_{(g,\C)\sim\mD_t}[g(\C\y^*)]$.
Moreover, letting  $\sigma$ denote the Hoffman constant of $\bar{\C}$ w.r.t. the polytope $\mP$, we have that $\forall \x\in\mP, ~\x^*\in\mX^*: ~ \dist(\x,\mX^*)^2 \leq \sigma^{-1}\Vert{\bar{\C}(\x-\x^*)}\Vert^2$.
\end{lemma}

\begin{proof}
Consider a distribution $\mD$ over $(\reals^m\rightarrow\reals)\times\reals^{m\times d}$, described by the following sampling procedure: pick $t\in[T]$ uniformly at random, and then sample $(g,\C)\sim\mD_t$. Clearly, it holds that
\begin{eqnarray*}
\forall \x\in\mP: ~ F(\x) &=& \frac{1}{T}\sum_{t=1}^T\E_{(g,\C)\sim\mD_t}[g(\C\x)] \\
&=& \E_{t\sim\textrm{Uni}[T]}[\E_{(g,\C)\sim\mD_t}[g(\C\x)]] \\
&=& \E_{(g,\C)\sim\mD}[g(\C\x)].
\end{eqnarray*}
Thus, it follows that
$\mX^* = \argmin_{\x\in\mP}\E_{(g,\C)\sim\mD}[g(\C\x)]$.

Note that since each distribution $\mD_t$ satisfies Assumption \ref{ass:ogd:dist}, it also satisfies  the assumptions of Lemma \ref{lem:hoffmanExpectation}.
It can be easily verified that as a consequence, the distribution $\mD$ also satisfies the assumptions of Lemma  \ref{lem:hoffmanExpectation}, and thus there exists a matrix $\bar{\C}$, satisfying $\bar{\C}^{\top}\bar{\C} = \E_{(g,\C)\sim\mD}[\C^{\top}\C]$, such that $\forall \x^*,\y^*\in\mX^*$:
\begin{eqnarray*}
\Vert{\bar{\C}(\x^*-\y^*)}\Vert^2 = (\x^*-\y^*)^{\top}\bar{\C}^{\top}\bar{\C}(\x^*-\y^*) =  0.
\end{eqnarray*}
By the definitions of $\bar{\C}$ and the distribution $\mD$, it holds that
\begin{eqnarray*}
\bar{\C}^{\top}\bar{\C} &=& \E_{(g,\C)\sim\mD}[\C^{\top}\C] \\
&=& \E_{t\sim\textrm{Uni}[T]}\left[{\E_{(g,\C)\sim\mD_t}[\C^{\top}\C]}\right] = \frac{1}{T}\sum_{t=1}^T\bar{\C}_t^{\top}\bar{\C}_t,
\end{eqnarray*}
where we define $\bar{\C}_t$ to be a matrix satisfying $\bar{\C}_t^{\top}\bar{\C}_t = \E_{(g,\C)\sim\mD_t}[\C^{\top}\C]$.

Thus, we have that
\begin{align*}
\forall \x^*,\y^*\in\mX^*: ~ 0 &= (\x^*-\y^*)^{\top}\bar{\C}^{\top}\bar{\C}(\x^*-\y^*) \\
&= \frac{1}{T}\sum_{t=1}^T(\x^*-\y^*)^{\top}\bar{\C}_t^{\top}\bar{\C}_t(\x^*-\y^*),
\end{align*}
which, since each $\bar{\C}_t^{\top}\bar{\C}_t$ is positive semidefinite, implies that $\forall \x^*,\y^*\in\mX^*,~t\in[T]$:
\begin{align}\label{eq:lem:uniformMinimizers:1}
  (\x^*-\y^*)^{\top}\bar{\C}_t^{\top}\bar{\C}_t(\x^*-\y^*) = 0.
\end{align}

Thus, fixing some $\x^*,\y^*\in\mX^*$ and $t\in[T]$, it holds that
\begin{align*}
&\E_{(g,\C)\sim\mD_t}[g(\C\x^*) - g(\C\y^*)]  \underset{(a)}{\leq}\\
&  \E_{(g,\C)\sim\mD_t}[G\Vert{\C(\x^*- \y^*)}\Vert]= \\
& G\sqrt{\left({\E_{(g,\C)\sim\mD_t}[\Vert{\C(\x^*- \y^*)}\Vert]}\right)^2} \leq \\
& G\sqrt{\E_{(g,\C)\sim\mD_t}[\Vert{\C(\x^*- \y^*)}\Vert^2]} =\\
& G\sqrt{\E_{(g,\C)\sim\mD_t}[(\x^*- \y^*)^{\top}\C^{\top}\C(\x^*- \y^*)]} =\\
& G\sqrt{(\x^*- \y^*)^{\top}\E_{(g,\C)\sim\mD_t}[\C^{\top}\C](\x^*- \y^*)} =\\
& G\sqrt{(\x^*- \y^*)^{\top}\bar{\C_t}^{\top}\bar{\C}_t(\x^*- \y^*)} \underset{(b)}{=} 0,
\end{align*}
where (a) follows via Eq. \eqref{eq:Lip} since each $g(\cdot)$ in the support of $\mD_t$ is convex and $G$-Lipschitz for some finite $G>0$, and (b) follows from Eq. \eqref{eq:lem:uniformMinimizers:1}. Thus, the first part of the lemma follows.

The second part of the lemma is a straightforward consequence of Lemma \ref{lem:hoffmanExpectation}, when applied to the distribution $\mD$, defined above.
\end{proof}

\begin{proof}[Proof of Theorem \ref{thm:ogd:master}]
Let us denote the set of of minimizers in hindsight: $\mX^* = \argmin_{\x\in\mP}\frac{1}{T}\sum_{t=1}^T\E_{f_t\sim\mD_t}\left[{f_t(\x)}\right]$.

Given the sequence of points generated by Algorithm \ref{alg:ogd} $\{\x_t\}_{t=1}^T$,  we define the sequence $\{\x_t^*\}_{t=1}^T$ as follows: 
$\forall t\geq 1$: $\x_t^* := \argmin_{\x\in\mX^*}\Vert{\x-\x_t}\Vert^2$,
i.e., $\x_t^*$ is the projection of $\x_t$ onto the set of optimal plays in hindsight $\mX^*$. 

Let us fix some $\x^*\in\mX^*$. By an application of Lemma \ref{lem:uniformMinimizers}, it holds that
{\small
\begin{align*}
\max_{\x\in\mP}\E\left[{\sum_{t=1}^Tf_t(\x_t) - f_t(\x)}\right] &= \E\left[{\sum_{t=1}^T f_t(\x_t) - f_t(\x^*)}\right] \\
&= \E\left[{\sum_{t=1}^T f_t(\x_t) - f_t(\x_t^*)}\right].
\end{align*}}

Thus, to prove the theorem, it suffices to upper bound $R_T := \E_{f_1\dots f_T}\left[{\sum_{t=1}^Tf_t(\x_t) - \sum_{t=1}^Tf_t(\x_t^*)}\right]$.

Let us also define the sequence $\{r_t := f_t(\x_t)-f_t(\x_t^*)\}_{t=1}^T$. Throughout the rest of the proof we write $\nabla_t$ as a short notation for $\nabla{}f_t(\x_t)
$.

As standard in the analysis of Online Gradient Descent, for every $t\geq 1$, we have that
\vspace{-5pt}
{\small
\begin{align*}
\Vert{\x_{t+1} - \x_t^*}\Vert^2 &\leq  \Vert{\y_{t+1} - \x_t^*}\Vert^2 = \Vert{\x_t - \eta_t\nabla_t - \x_t^*}\Vert^2 \\
&= \Vert{\x_t-\x_t^*}\Vert^2 - 2\eta_t(\x_t-\x_t^*)^{\top}\nabla_t + \eta_t^2\Vert{\nabla_t}\Vert^2,
\end{align*}}
where the first inequality holds since $\x_{t+1}$ is the orthogonal projection of $\y_{t+1}$ onto $\mP$.
Rearranging and recalling that $f_t(\x) = g_t(\C_t\x)$ for some $g_t,\C_t$, we have that
{\small
\begin{align*}
(\C_t(\x_t-\x_t^*))^{\top}\nabla{}g_t(\C_t\x_t) =& (\x_t - \x_t^*)^{\top}\nabla_t  \leq  \frac{1}{2\eta_t}\Vert{\x_t-\x_t^*}\Vert^2 \\
&- \frac{1}{2\eta_t}\Vert{\x_{t+1}-\x_t^*}\Vert^2 + \frac{\eta_t}{2}\Vert{\nabla_t}\Vert^2.
\end{align*}
}
Since $g_t(\cdot)$ is $\alpha_1$-strongly convex, we have
{\small
\begin{align*}
r_t  = g_t(\C_t\x_t) - g_t(\C_t\x_t^*) \leq & \frac{1}{2\eta_t}\Vert{\x_t-\x_t^*}\Vert^2 - \frac{1}{2\eta_t}\Vert{\x_{t+1}-\x_t^*}\Vert^2  \\
&-\frac{\alpha_1}{2}\Vert{\C_t(\x_t-\x_t^*)}\Vert^2  + \frac{\eta_t}{2}\Vert{\nabla_t}\Vert^2 .
\end{align*}}

By the definition of the sequence $\{\x_t^*\}_{t=1}^T$, we have
\begin{align*}
r_t \leq &\frac{1}{2\eta_t}\dist(\x_t,\mX^*)^2 - \frac{1}{2\eta_t}\dist(\x_{t+1},\mX^*)^2  \\
& - \frac{\alpha_1}{2}\Vert{\C_t(\x_t-\x_t^*)}\Vert^2  + \frac{\eta_t\Vert{\nabla_t}\Vert^2}{2}.
\end{align*}

Summing over all $T$ iterations, rearranging and taking expectation on both sides, we have that {\small
\begin{align}\label{eq:ogd:master:1}
R_T &\leq  \E_{f_1\dots f_T}\Bigg[\frac{1}{2\eta_1}\dist(\x_1,\mX^*)^2 + \frac{\bar{G}^2}{2}\sum_{t=1}^T\eta_t \nonumber \\
&+ \frac{1}{2}\sum_{t=2}^T\Big(\frac{1}{\eta_t}\dist(\x_t,\mX^*)^2 - \frac{1}{\eta_{t-1}}\dist(\x_t,\mX^*)^2 \nonumber \\
& -\alpha_1\Vert{\C_t(\x_t - \x_t^*)}\Vert^2\Big)\Bigg]\nonumber \\
& \leq \frac{1}{2}\E_{f_1\dots f_t}\Bigg[\sum_{t=2}^T\left({\frac{1}{\eta_t} - \frac{1}{\eta_{t-1}}}\right)\dist(\x_t,\mX^*)^2 \nonumber\\
&- \alpha_1\Vert{\C_t(\x_t - \x_t^*)}\Vert^2\Bigg] + \frac{D^2}{2\eta_1}+ \frac{\bar{G}^2}{2}\sum_{t=1}^T\eta_t,
\end{align}}
where we have used the fact that $\forall t\in[T]:$ $\E_{\mD_t}[\Vert{\nabla{}f_t(\x_t)}\Vert^2] \leq \bar{G}^2$.

Note that since for all $t$, $\C_t$ is independent of $\x_t,\x_t^*$, we have that
{\small
\begin{align}\label{eq:ogd:master:2}
&\E_{f_1\dots f_t}\Bigg[\left({\frac{1}{\eta_t} - \frac{1}{\eta_{t-1}}}\right)\dist(\x_t,\mX^*)^2 - \alpha_1\Vert{\C_t(\x_t - \x_t^*)}\Vert^2\Bigg]  =\nonumber\\
&\E_{f_1\dots f_{t-1}}\Bigg[\left({\frac{1}{\eta_t} - \frac{1}{\eta_{t-1}}}\right)\dist(\x_t,\mX^*)^2 \nonumber \\
&- \alpha_1(\x_t - \x_t^*)^{\top}\E_{f_t}[\C_t^{\top}\C_t](\x_t - \x_t^*)\Bigg]= \nonumber\\
&\E_{f_1\dots f_{t-1}}\Bigg[\left({\frac{1}{\eta_t} - \frac{1}{\eta_{t-1}}}\right)\dist(\x_t,\mX^*)^2 - \alpha_1\Vert{\bar{\C_t}(\x_t - \x_t^*)}\Vert^2\Bigg],  
\end{align}}
where we let $\bar{\C}_t$ be such that $\bar{\C}_t^{\top}\bar{\C}_t = \E_{(g,\C_t)\sim\mD_t}[\C_t^{\top}\C_t]$.

By Lemma \ref{lem:uniformMinimizers} and the assumption of the theorem, it holds that
\begin{eqnarray*}
\forall t\in[T]: \quad \dist(\x_t,\mX^*)^2 &\leq &\sigma^{-1}\Vert{\bar{\C}(\x_t-\x_t^*)}\Vert^2 \\
&\leq &(\alpha_2\sigma)^{-1}\Vert{\bar{\C}_t(\x_t-\x_t^*)}\Vert^2.
\end{eqnarray*}
Thus, since the step-size $\eta_t$ is monotonically non-increasing with $t$, we have that $\forall t\geq 2$:
{\small
\begin{eqnarray}\label{eq:ogd:master:3}
&&\left({\frac{1}{\eta_t} - \frac{1}{\eta_{t-1}}}\right)\dist(\x_t,\mX^*)^2 - \alpha_1\Vert{\bar{\C_t}(\x_t - \x_t^*)}\Vert^2 \nonumber \\
&& \leq \left({\left({\frac{1}{\eta_t} - \frac{1}{\eta_{t-1}}}\right)\frac{1}{\alpha_2\sigma} - \alpha_1}\right)\Vert{\bar{\C}_t(\x_t - \x_t^*)}\Vert^2.
\end{eqnarray}
}

Combining Eq. \eqref{eq:ogd:master:1}, \eqref{eq:ogd:master:2}, \eqref{eq:ogd:master:3}, and plugging-in our choice of step-size $\eta_t = \frac{1}{\alpha_1\alpha_2\sigma{}t}$, we conclude that
{\small
\begin{eqnarray*}
R_T \leq \frac{D^2}{2\eta_1}+ \frac{\bar{G}^2}{2}\sum_{t=1}^T\eta_t 
\leq \frac{\alpha_1\alpha_2\sigma{}D^2}{2} + \frac{\bar{G}^2}{2\alpha_1\alpha_2\sigma}(1+ \ln{T}).
\end{eqnarray*}}
\end{proof}

\begin{figure*}[h]
    \centering
          \begin{subfigure}{0.3\textwidth}
        \includegraphics[width=1.8in]{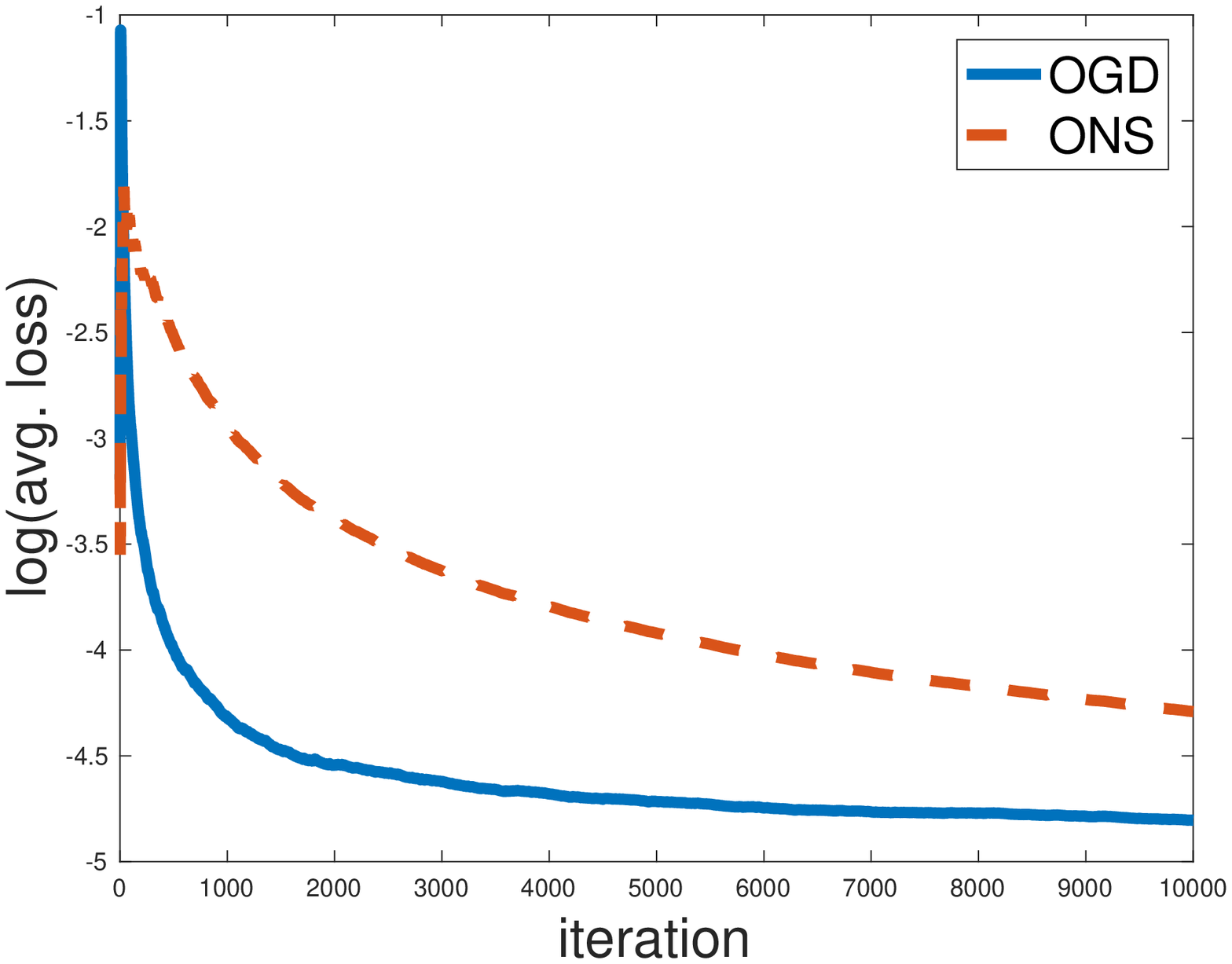}
        \caption{synthetic}
      \end{subfigure}    
    \begin{subfigure}{0.3\textwidth}
        \includegraphics[width=1.8in]{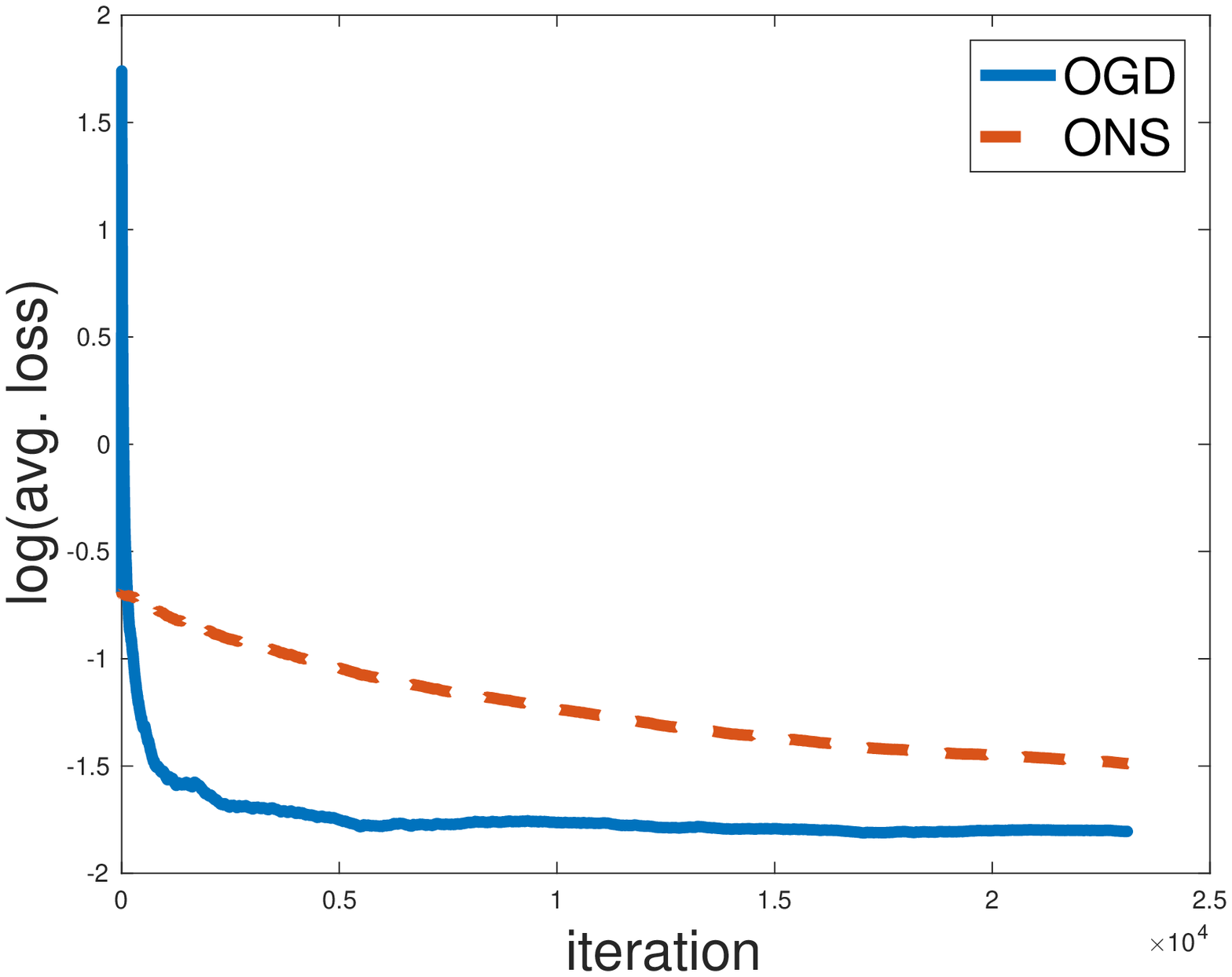}
        \caption{MNIST}
     \end{subfigure}
    \begin{subfigure}{0.3\textwidth}
        \includegraphics[width=1.8in]{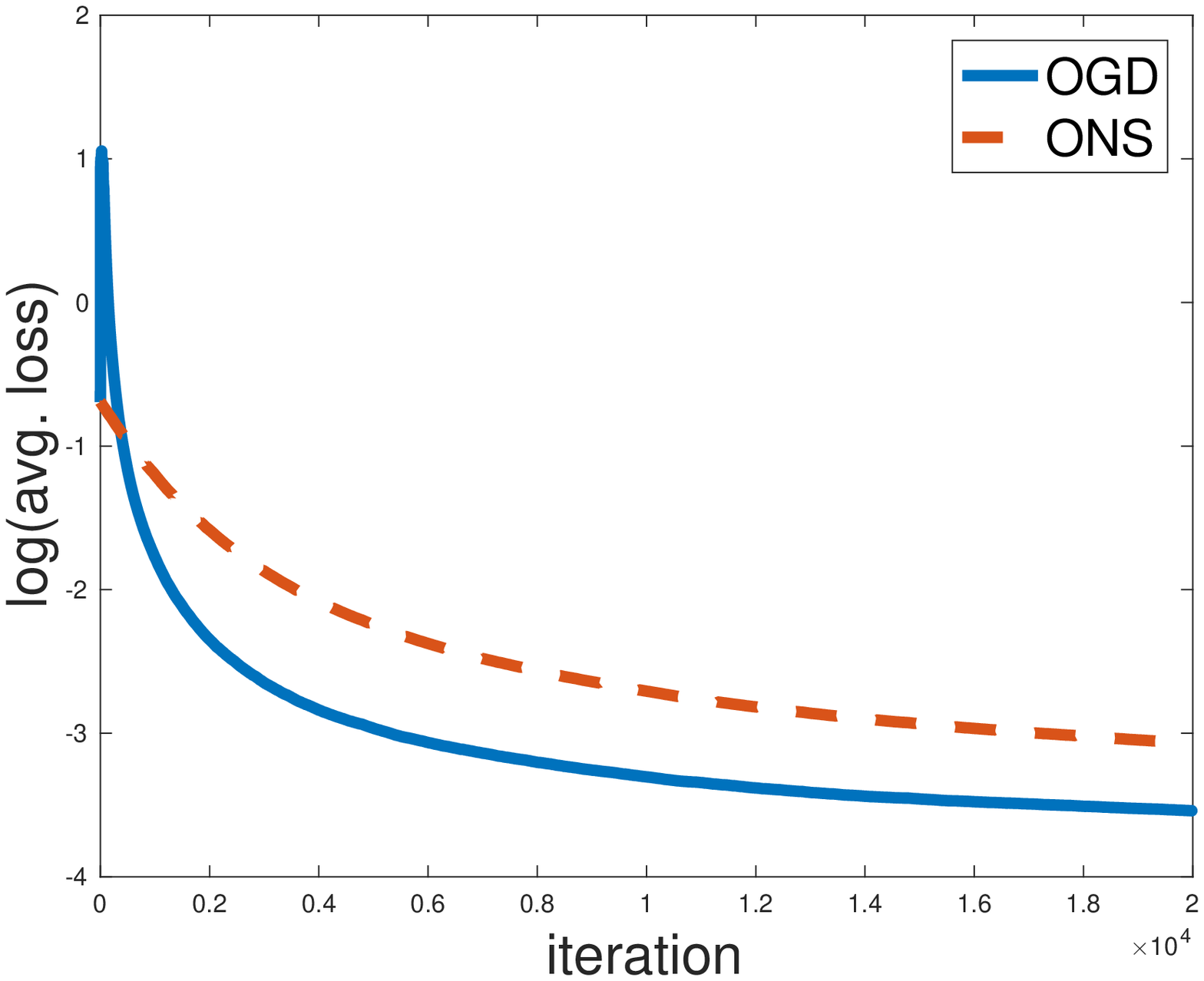}
        \caption{CIFAR10}
     \end{subfigure}
    \caption{Comparing the (log) average loss vs. number of iterations for OGD and ONS.}\label{fig:1}
\end{figure*}

\subsection{Applications of Theorem \ref{thm:ogd:master}}\label{subsec:app}
\paragraph{Deterministic data:}
In case  $g_1,\dots g_T:\reals^m\rightarrow\reals$ are arbitrary $\alpha_1$-strongly convex and differentiable functions over $\reals^m$, then a sufficient condition on the matrices $\C_1,\dots,\C_T$ for applying the result of Theorem \ref{thm:ogd:master} is that there exists a positive constant $\alpha_2$ such that
\begin{eqnarray}\label{eq:deterministicData}
\forall (i,j)\in[T]\times[T]: \qquad \C_i^{\top}\C_i \succeq \alpha_2\C_j^{\top}\C_j,
\end{eqnarray}
or in a different formulation: $\textrm{row-span}(\C_1) = \textrm{row-span}(\C_2) = \dots = \textrm{row-span}(\C_T)$.

A simple application of Theorem \ref{thm:ogd:master} upper-bounds the regret by $\frac{\alpha_1\alpha_2\sigma{}D^2}{2} + \frac{G^2}{2\alpha_1\alpha_2\sigma}(1 + \ln{T})$,
where $\sigma$ is the Hoffman constant of $\frac{1}{T}\sum_{t=1}^T\C_t^{\top}\C_t$ w.r.t. the polytope $\mP$.

\paragraph{Semi-adversarial data:}
 A way to circumvent the limitation of condition \eqref{eq:deterministicData}, is to consider slightly ``easier" data. In particular if we let $\{g_t(\cdot)\}_{t\in[T]}$ be as in the deterministic case, but we assume that the matrices $\C_1,\dots,\C_T$ are perturbed realizations of some underlying deterministic sequence $\tilde{\C}_1,\dots,\tilde{\C}_T$. That is, we let $\tilde{\C}_t$ be arbitrary, but the observed matrix $\C_t$ is a perturbed version given by $\C_t = \tilde{\C}_t + \N_t$, where $\N_t\sim\mD$, where $\mD$ is a fixed unknown distribution. Then, the condition in \eqref{eq:deterministicData} could be easily replaced by the requirement: $\forall t\in[T]$,  $\E_{\mD}[\N] = \mathbf{0}$ and $\E_{\mD}[\N^{\top}\N] \succeq \alpha_2\tilde{\C}_t^{\top}\tilde{\C}_t$, where $\alpha_2$, as before, is a positive constant.

An application of Theorem \ref{thm:ogd:master} upper-bounds the expected regret by $\frac{\alpha_1\alpha_2\sigma{}D^2}{2} + \frac{\bar{G}^2}{2\alpha_1\alpha_2\sigma}(1 + \ln{T})$,
where $\sigma$ is the Hoffman constant of the expected matrix $\frac{1}{T}\sum_{t=1}^T\C_t^{\top}\C_t + \E_{\mD}[\N^{\top}\N]$, w.r.t. the polytope $\mP$.
As mentioned in Section \ref{sec:examples}, such a setting underlies for instance the \textit{Universal Linear Filtering} problem studied in \cite{Moon2009, GarberH13}.

\paragraph{Shifting stochastic data:} Assuming $(g_t(\cdot),\C_t)$ is sampled out of a distribution $\mD_t$ (possibly changing from round to round), under assumptions on $g_t(\cdot)$ as above, a sufficient condition on the stochastic matrices $\C_t$ for applying the result of Theorem \ref{thm:ogd:master} is that there exists some $\alpha_2>0$ such that for all $t_1,t_2\in[T]$: $\E_{\C_1\sim\mD_{t_1}}[\C_1^{\top}\C_1] \succeq \alpha_2\E_{\C_2\sim\mD_{t_2}}[\C_2^{\top}\C_2]$, in which case we get the same bound as in the ``semi-adversarial" case with $\sigma$ being the Hoffman constant of $\frac{1}{T}\sum_{t=1}^T\E_{\C\sim\mD_t}[\C^{\top}\C]$ w.r.t. $\mP$.

\section{Experiments}

In this section we provide empirical evidence for the performance of Online Gradient Descent on curved, though not strongly convex, losses. Since the computational advantage of OGD over competing methods is clear, we focus on demonstrating convergence in terms of the average loss.
We consider the LASSO optimization problem, i.e., the loss function on each round $t$ is $f_t(\x) := \frac{1}{2}\Vert{\a_t^{\top}\x-b_t}\Vert^2$ and the feasible polytope is an $\ell_1$ ball. In all experiments we compare OGD with Online Newton Step (ONS).

\textbf{Synthetic data:} We compare OGD and ONS in an online stochastic setting.
We fix the dimension to $d=100$ and generate a random PSD matrix $\M\in\reals^{d\times d}$ with rank = $50$ and with decaying eigenvalues given by $\lambda_i = 10\cdot{}0.8^{i-1}$ for all $i\in[50]$. We set $\a_t = \v_t^{\top}\M$ for a random unit vector $\v_t$, and $b_t := \a_t^{\top}\w^*+0.1n_t$, where $\w^*$ is a fixed sparse vector chosen at random, and $n_t\sim\mathcal{N}(0, 1)$. The radius of the feasible $\ell_1$ ball is set to $r=10$.
Since determining the Hoffman constant is difficult in general, we heuristically set $\sigma$ according the eigenvalue of the covariance matrix $\A = \frac{1}{T}\sum_{t=1}^T\a_t\a_t^{\top}$ which corresponds to the numerical rank of $\A$. We use this choice in all of our experiments which seems to work well.
 and set the step-size accordingly to $\eta_t = \frac{1}{\sigma{}t}$ (note $g_t(\cdot)$ in our case is $1$-strongly convex). ONS is implemented as suggested in \cite{Hazan16}. For both methods we plot the (log) average loss vs. number of iterations (we use $T=10000$). 

\textbf{MNIST data:} Next we experiment with the MNIST handwritten digit recognition dataset \cite{Lecun1998}. Specifically, we use the training dataset, keeping only the data related to digits $3,5$. We  set $b_t$ by assigning value $1$ to instances corresponding to the digit $5$ and $-1$ to those corresponding to $3$. Finally, in order to increase the amount of data, we replicate the data and concatenate twice. We set the radius of the feasible $\ell_1$ ball to $r=5$.

\textbf{CIFAR10 data:} we use the CIFAR10 tiny image dataset \cite{cifar-10} which contains 50000 32x32 images in RGB format. We convert the images to grayscale and keep only the data related to the classes "automobile" and "truck", assigning the first the label $b_t = -1$ and the second the label $b_t = 1$. Here we also replicate the data twice and set $r=5$.

The results for all datasets are presented in Figure \ref{fig:1}. It is clearly observable that in all three cases OGD is comparable to ONS in terms of regret and even far better.



\section{Acknowledgments}
This research was supported by the ISRAEL SCIENCE FOUNDATION (grant No. 1108/18).

\bibliography{bib}

\begin{thebibliography}{10}

\bibitem{Beck15}
Amir Beck and Shimrit Shtern.
\newblock Linearly convergent away-step conditional gradient for non-strongly
  convex functions.
\newblock {\em Math. Program.}, 164(1-2):1--27, 2017.

\bibitem{boyd2004convex}
Stephen Boyd and Lieven Vandenberghe.
\newblock {\em Convex optimization}.
\newblock Cambridge university press, 2004.

\bibitem{GarberH13}
Dan Garber and Elad Hazan.
\newblock Adaptive universal linear filtering.
\newblock {\em {IEEE} Trans. Signal Processing}, 61(7):1595--1604, 2013.

\bibitem{Guler2010}
Osman G{\"u}ler.
\newblock {\em Foundations of optimization}, volume 258.
\newblock Springer Science \& Business Media, 2010.

\bibitem{Hazan16}
Elad Hazan.
\newblock Introduction to online convex optimization.
\newblock {\em Foundations and Trends in Optimization}, 2(3-4):157--325, 2016.

\bibitem{AHK07}
Elad Hazan, Amit Agarwal, and Satyen Kale.
\newblock Logarithmic regret algorithms for online convex optimization.
\newblock {\em Machine Learning}, 69(2-3):169--192, 2007.

\bibitem{Hazan14}
Elad Hazan and Satyen Kale.
\newblock Beyond the regret minimization barrier: optimal algorithms for
  stochastic strongly-convex optimization.
\newblock {\em Journal of Machine Learning Research}, 15(1):2489--2512, 2014.

\bibitem{Hoffman52}
Alan~J Hoffman.
\newblock On approximate solutions of systems of linear inequalities.
\newblock {\em Journal of Research of the National Bureau of Standards}, 49(4),
  1952.

\bibitem{Karimi16}
Hamed Karimi, Julie Nutini, and Mark Schmidt.
\newblock {\em Linear Convergence of Gradient and Proximal-Gradient Methods
  Under the Polyak-{\L}ojasiewicz Condition}, pages 795--811.
\newblock Springer International Publishing, Cham, 2016.

\bibitem{cifar-10}
Alex Krizhevsky.
\newblock Learning multiple layers of features from tiny images, 2009.

\bibitem{Lecun1998}
Yann LeCun, L{\'e}on Bottou, Yoshua Bengio, and Patrick Haffner.
\newblock Gradient-based learning applied to document recognition.
\newblock {\em Proceedings of the IEEE}, 86(11):2278--2324, 1998.

\bibitem{Luo6}
Haipeng Luo, Alekh Agarwal, Nicol{\`{o}} Cesa{-}Bianchi, and John Langford.
\newblock Efficient second order online learning by sketching.
\newblock In {\em Advances in Neural Information Processing Systems 29: Annual
  Conference on Neural Information Processing Systems 2016, December 5-10,
  2016, Barcelona, Spain}, pages 902--910, 2016.

\bibitem{Moon2009}
Taesup Moon and Tsachy Weissman.
\newblock Universal fir mmse filtering.
\newblock {\em IEEE Transactions on Signal Processing}, 57(3):1068--1083, 2009.

\bibitem{Necoara16}
Ion Necoara, Yu~Nesterov, and Francois Glineur.
\newblock Linear convergence of first order methods for non-strongly convex
  optimization.
\newblock {\em Mathematical Programming}, pages 1--39, 2016.

\bibitem{SSS12}
Shai Shalev{-}Shwartz.
\newblock Online learning and online convex optimization.
\newblock {\em Foundations and Trends in Machine Learning}, 4(2):107--194,
  2012.

\bibitem{Sion58}
Maurice Sion.
\newblock On general minimax theorems.
\newblock {\em Pacific Journal of mathematics}, 8(1):171--176, 1958.

\bibitem{Lasso96}
Robert Tibshirani.
\newblock Regression shrinkage and selection via the lasso.
\newblock {\em Journal of the Royal Statistical Society. Series B
  (Methodological)}, pages 267--288, 1996.

\bibitem{wang2014}
Po-Wei Wang and Chih-Jen Lin.
\newblock Iteration complexity of feasible descent methods for convex
  optimization.
\newblock {\em Journal of Machine Learning Research}, 15(1):1523--1548, 2014.

\bibitem{Xu17}
Yi~Xu, Qihang Lin, and Tianbao Yang.
\newblock Stochastic convex optimization: Faster local growth implies faster
  global convergence.
\newblock In {\em International Conference on Machine Learning}, pages
  3821--3830, 2017.

\bibitem{Zinkevich03}
Martin Zinkevich.
\newblock Online convex programming and generalized infinitesimal gradient
  ascent.
\newblock In {\em Machine Learning, Proceedings of the Twentieth International
  Conference {(ICML} 2003), August 21-24, 2003, Washington, DC, {USA}}, pages
  928--936, 2003.

\end{thebibliography}
\bibliographystyle{plain}






\onecolumn
\pagebreak

\appendix

\section{Proof of Hoffman's Lemma}

For the convenience of the reader we first restate the lemma.
\begin{lemma}
Let $\mP:=\{\x\in\reals^d ~ | ~ \A\x \leq \b\}$  be a compact and convex polytope and let $\C\in\reals^{m\times d}$. Given a vector $\c\in\reals^m$, define the set $\mP(\C,\c) := \{\x\in\mP ~ | ~ \C\x=\c\}$. If $\mP(\C,\c) \neq \emptyset$, then there exists $\sigma > 0$ such that
\begin{eqnarray*}
\forall \x\in\mP: \qquad \dist(\x,\mP(\C,\c))^2\leq \sigma^{-1}\Vert{\C\x - \c}\Vert^2.
\end{eqnarray*}
Moreover, we have the bound $\sigma \geq \min_{\Q\in\mathcal{M}}\lambda_{\min}\left({\Q\Q^{\top}}\right)$, where $\mathcal{M}$ is the set of all $d\times d$ matrices whose rows are linearly independent rows of the matrix $\M := (\A^{\top},~ \C^{\top})^{\top}$, and $\lambda_{\min}(\cdot)$ denotes the smallest non-zero eigenvalue.
\end{lemma}

\begin{proof}
The proof is based on the proof in \cite{Guler2010} (pages 299-301), though some details are different.

Let us  define the matrix and vector
\begin{eqnarray*}
\tilde{\A} := \left( \begin{array}{c}
~\A \\
~\C \\
-\C \end{array} \right), \qquad
\tilde{\b} := \left( \begin{array}{c}
~\b \\
~\c \\
-\c \end{array} \right).
\end{eqnarray*}
Note that the following equivalence holds trivially
\begin{eqnarray*}
\mP(\C,\c) := \{\x\in\reals^d ~ | ~ \tilde{\A}\x \leq \tilde{\b}\}.
\end{eqnarray*}

We can now write,

\begin{eqnarray}\label{eq:hoff:1}
\dist(\x,\mP(\C,\c)) &=& \min_{\y\in\mP(\C,\c)}\Vert{\x-\y}\Vert \nonumber \\
&=& \min_{\tilde{\A}\y\leq\tilde{\b}}\max_{\Vert{\u}\Vert \leq 1}(\x-\y)^{\top}\u \nonumber\\
&\underset{(a)}{=}& \max_{\Vert{\u}\Vert \leq 1}\min_{\tilde{\A}\y\leq\tilde{\b}}(\x-\y)^{\top}\u \nonumber\\
&=& \max_{\Vert{\u}\Vert \leq 1}\min_{\tilde{\A}\w\geq\hat{\b}}\w^{\top}\u \qquad \{\w:=\x-\y, \quad \hat{\b} := \tilde{\A}\x - \tilde{\b}\} \nonumber\\
&\underset{(b)}{=}& \max_{\Vert{\u}\Vert \leq 1}\max\{\q^{\top}\hat{\b} ~ | ~ \tilde{\A}^{\top}\q = \u,  \q \geq 0\} \nonumber\\
&=& \max\{\q^{\top}\hat{\b} ~ | ~ \Vert{\tilde{\A}^{\top}\q}\Vert \leq 1,  \q \geq 0\}.
\end{eqnarray}
where (a) follows from minimax duality (which holds since the polytope $\tilde{\A}\y\leq \tilde{\b}$, by construction, is compact), see \cite{Sion58}, and (b) follows from linear programming duality.

Let $\q^*$ be an optimal solution to the RHS of Eq. \eqref{eq:hoff:1}, and let $\u^*$ be a corresponding vector such that $\tilde{\A}^{\top}\q^* = \u^*$.

We can now write
\begin{eqnarray}\label{eq:hoff:2}
\dist(\x,\mP(\C,\c)) &=& \max\{\q^{\top}\left({\tilde{\A}\x - \tilde{\b}}\right) ~ | ~ \tilde{\A}^{\top}\q = \u^*,  \q \geq 0\}.
\end{eqnarray}

Note that the RHS of Eq. \eqref{eq:hoff:2} is again a linear program with the feasible polytope being $\hat{\mP}:=\{\q ~ | ~ \tilde{\A}^{\top}\q = \u^*, ~ \q \geq 0\}$. Hence, an optimal solution to RHS of \eqref{eq:hoff:2} $\hat{\q}^*$ is without loss of generality a vertex of $\hat{\mP}$. Hence, the non-zero entries of $\hat{\q}^*$ correspond to a set of linearly independent columns of $\tilde{\A}^{\top}$. That is, there exists a matrix $\M^*$, whose rows are taken from $\tilde{\A}$ and are linearly independent (i.e., $\M^*\M^{*\top}$ is positive definite), such that 
\begin{eqnarray*}
1 \geq \Vert{\u^*}\Vert= \Vert{\tilde{\A}^{\top}\hat{\q}^*}\Vert  \geq \Vert{\hat{\q}^*}\Vert\cdot\sigma_{\min}(\M^*) = \Vert{\hat{\q}^*}\Vert\cdot\sqrt{\lambda_{\min}(\M^*\M^{*\top})}.
\end{eqnarray*}
Thus, we have that
\begin{eqnarray}\label{eq:hoff:3}
\Vert{\hat{\q}^*}\Vert \leq \frac{1}{\sqrt{\lambda_{\min}(\M^{*}\M^{*\top})}}.
\end{eqnarray}

Note that since by definition $\hat{\q}^* \geq 0$, using the Cauchy-Schwarz inequality we have that 
\begin{eqnarray}\label{eq:hoff:4}
\textrm{RHS of \eqref{eq:hoff:2} } \leq \Vert{\hat{\q}^*}\Vert\cdot \Vert{(\tilde{\A}\x-\tilde{\b})_+}\Vert \leq \Vert{\hat{\q}^*}\Vert\cdot \Vert{\C\x-\c}\Vert,
\end{eqnarray}
were for any vector $\w$ we let $\w_+$ denote the vector that corresponds to the subset of non-negative entries of $\w$. Note that the last inequality holds since $\x$ is feasible with respect to the polytope $\mP$ (hence $\A\x - \b \leq 0$), and since for each row $i$ of $\C$, it trivially holds that either $\C_i\x - \c_i \leq 0$, or $-\C_i\x + \c_i < 0$. 

Combining Eq. \eqref{eq:hoff:3}, \eqref{eq:hoff:4}, yields the lemma.
\end{proof}

\end{document}